\documentclass[twoside]{article}
\usepackage{ecj,palatino,epsfig,latexsym}

\usepackage{subfigure}
\usepackage{graphicx}

\usepackage{amsmath,amssymb,amsthm}
\usepackage{amstext}
\usepackage{dsfont}
\usepackage{xspace}

\usepackage[vlined,linesnumbered,titlenumbered,algoruled]{algorithm2e}%
\usepackage{algorithmic}
\usepackage{color}

\usepackage{pdflscape}
\usepackage{multirow}
\usepackage{rotating}
\usepackage{cite}

\newtheorem{theorem}{Theorem}
\newtheorem{lemma}{Lemma}
\newtheorem{corollary}{Corollary}

\newcommand{\IF}{\textbf{if}\xspace}
\newcommand{\THEN}{\textbf{then}\xspace}

\newcommand{\REPEAT}{\textbf{repeat}\xspace}

\renewcommand{\epsilon}{\varepsilon}
\newcommand{\R}{\mathds{R}}
\newcommand{\N}{\mathds{N}}

\newcommand{\E}[1]{\text{E}\left(#1\right)}

\newcommand{\Prob}[1]{\text{Prob}\left(#1\right)}
\newcommand{\Oh}[1]{\mathord{O}\mathord{\left(#1\right)}}

\newcommand{\LO}{\textup{LO}\xspace}

\newcommand{\OM}{\textup{OneMax}\xspace}
\newcommand{\Jump}{\textup{Jump}\xspace}

\newcommand{\EA}{\text{(1+1)~EA}\xspace}

\newcommand{\ie}{i.\,e.\xspace}

\newcommand{\migint}{\tau}
\newcommand{\paralleltime}{T^{\mathrm{par}}}
\newcommand{\sequentialtime}{T^{\mathrm{seq}}}
\newcommand{\communication}{T^{\mathrm{com}}}
\newcommand{\pinform}{p}
\newcommand{\pcomm}{\pinform}
\newcommand{\Tinform}[1]{\xi\left(
#1\right)}

\newcommand{\ignore}[1]{\ensuremath{}}

\DeclareMathOperator{\diam}{diam}

\definecolor{black}{rgb}{0.0,0.0,0.0}

\usepackage{tikz}
\usetikzlibrary{arrows,shapes}

\usepackage{pgfplots}
\usepackage{pgfplotstable}
\pgfplotsset{compat=1.3}

\definecolor{black}{rgb}{0.0,0.0,0.0}

\usepackage{tikz}
\usepackage{verbatim}
\usetikzlibrary{arrows,shapes}

\usepgfplotslibrary{external}

\begin{document}


\title{General Upper Bounds on the Running Time of~Parallel~Evolutionary~Algorithms\thanks{\ A preliminary version of this paper with parts of the results was published at PPSN~2010~\cite{Lassig2010a}.}}

\author{\name{\bf J{\"o}rg L{\"a}ssig} \hfill
\addr{jlaessig@hszg.de}\\
        \addr{Department of Electrical Engineering and Computer Science, \newline University of Applied Sciences Zittau/G\"orlitz, Germany}
\AND
       \name{\bf Dirk Sudholt} \hfill
       \addr{d.sudholt@sheffield.ac.uk}\\
        \addr{Department of Computer Science,}
        \addr{University of Sheffield, United Kingdom}
}

\maketitle

\begin{abstract}
We present a new method for analyzing the running time of parallel evolutionary algorithms with spatially structured populations. Based on the fitness-level method, it yields upper bounds on the expected parallel running time. This allows to rigorously estimate the speedup gained by parallelization. Tailored results are given for common migration topologies: ring graphs, torus graphs, hypercubes, and the complete graph. Example applications for pseudo-Boolean optimization show that our method is easy to apply and that it gives powerful results. In our examples the possible speedup increases with the density of the topology. Surprisingly, even sparse topologies like ring graphs lead to a significant speedup for many functions while not increasing the total number of function evaluations by more than a constant factor. We also identify which number of processors yield asymptotically optimal speedups, thus giving hints on how to parametrize parallel evolutionary algorithms.
\end{abstract}


\section{Introduction}
\label{sec:Introduction}

Due to the increasing number of CPU cores, exploiting possible speedups by parallel computations is nowadays more important than ever.
Parallel evolutionary algorithms (EAs) form a popular class of heuristics with many applications to computationally expensive problems~\cite{Nedjah2006,Tomassini2005,Luque2011}. This includes \emph{island models}, also called \emph{distributed EAs}, \emph{multi-deme EAs} or \emph{coarse-grained EAs}. Evolution is parallelized by evolving subpopulations, called \emph{islands}, on different processors. Individuals are periodically exchanged in a process called \emph{migration}, where selected individuals, or copies of these, are sent to other islands, according to a migration topology that determines which islands are neighboring. Also more fine-grained models are known, where neighboring subpopulations communicate in every generation, first and foremost in \emph{cellular EAs}~\cite{Tomassini2005}.

By restricting the flow of information through spatial structures and/or infrequent communication, diversity in the whole system is increased. Researchers and practitioners frequently report that parallel EAs speed up the computation time, and at the same time lead to a better solution quality~\cite{Luque2011}.

Despite these successes, a long history~\cite{Cant`u-Paz1997} and very active research in this area~\cite{Rudolph2006,Alba2005,Luque2011}, the theoretical foundation of parallel EAs is still in its infancy. The impact of even the most basic parameters on performance is not well understood~\cite{Skolicki2005}. Past and present research is mostly empirical, and a solid theoretical foundation is missing. Theoretical studies are mostly limited to artificial settings. In the study of \emph{takeover times}, one asks how long it takes for a single optimum to spread throughout the whole parallel EA, if the EA uses only selection and migration, but neither mutation nor crossover~\cite{Rudolph2000a,Rudolph2006}. This gives a useful indicator for the speed at which communication is spread, but it does not give any formal results about the running time of evolutionary algorithms with mutation and/or crossover.

One way of gaining insight into the capabilities and limitations of parallel EAs is by means of rigorous running time analysis~\cite{Wegener2002}.
By asymptotic bounds on the running time we can compare different implementations of parallel EAs and assess the speedup gained by parallelization in a rigorous manner.

In~\cite{Lassig2010} the authors presented the first running time analysis of a parallel evolutionary algorithm with a non-trivial migration topology. It was demonstrated for a constructed problem that migration is essential in the following way. A suitably parametrized island model with migration has a polynomial running time while the same model without migration as well as comparable panmictic populations need exponential time, with overwhelming probability. Neumann, Oliveto, Rudolph, and Sudholt~\cite{Neumann2011} presented a similar result for island models using crossover. If islands perform crossover with immigrants during migration, this can drastically speed up optimization. This was demonstrated for a pseudo-Boolean example as well as for instances of the \textsc{VertexCover} problem~\cite{Neumann2011}.

In this work we take a broader view and consider the speedup gained by parallelization for various common pseudo-Boolean functions and function classes of varying difficulty. A general method is presented for proving upper bounds on the parallel running time of parallel EAs. The latter is defined as the number of generations of the parallel EA until a global optimum is found for the first time.
This allows us to estimate the speedup gained by parallelization, defined as the ratio of the expected parallel running time of an island model and the expected running time for a single island. It also can be used to determine how to choose the number of islands such that the parallel running time is reduced as much as possible, while still maintaining an asymptotically optimal speedup.

Our method is based on the \emph{fitness-level method} or \emph{method of $f$-based partitions}, a simple and well-known tool for the analysis of evolutionary algorithms~\cite{Wegener2002}. The main idea of this method is to divide the search space into sets $A_1, \dots, A_m$, strictly ordered according to fitness values of elements therein. Elitists EAs, i.\,e., EAs where the best fitness value in the population can never decrease, can only increase their current best fitness. If, for each set $A_i$ we know a lower bound $s_i$ on the probability that an elitist EA finds an improvement, i.\,e., for finding a new search point in a new best fitness-level set $A_{i+1} \cup \dots \cup A_m$, this gives rise to an upper bound $\sum_{i=1}^m 1/s_i$ on the expected running time. The method is described in more detail in Section~\ref{sec:Preliminaries}.

In Section~\ref{sec:general-upper-bounds} we first derive a general upper bound for parallel EAs, based on fitness levels.
Our general method is then tailored towards different spatial structures often used in fine-grained or cellular evolutionary algorithms and parallel architectures in general: ring graphs (Theorem~\ref{the:method-ring} in Section~\ref{sec:ring}), torus graphs (Theorem~\ref{the:method-torus} in Section~\ref{sec:torus}), hypercubes (Theorem~\ref{the:method-hypercube} in Section~\ref{sec:hypercube}) and complete graphs (Theorems~\ref{the:method-completegraph} and~\ref{the:method-completegraph-refined} in Section~\ref{sec:PerfectParallelization}).

The only assumption made is that islands run elitist algorithms, and that in each generation each island has a chance of transmitting individuals from its best current fitness level to each neighboring island, independently with probability at least $\pinform$. We call the latter the \emph{transmission probability}. It can be used to model various stochastic effects such as disruptive variation operators, the impact of selection operators, probabilistic migration, probabilistic emigration and immigration policies, and transient faults in the network. This renders our method widely applicable to a broad range of settings.

\subsection{Main Results}

Our estimates of parallel running times from Theorems~\ref{the:method-ring},~\ref{the:method-torus},~\ref{the:method-hypercube},~\ref{the:method-completegraph}, and~\ref{the:method-completegraph-refined} are summarized in the following theorem, hence characterizing our main results. Throughout this work $\mu$ always denotes the number of islands.
\begin{theorem}
\label{the:generalBounds}
Consider an island model with $\mu$ islands where each island runs an elitist EA.
For each island let there be a fitness-based partition $A_1, \dots, A_m$ such that for all $1 \le i < m$ all points in $A_i$ have a strictly worse fitness than all points in $A_{i+1}$, and $A_m$ contains all global optima. We say that an island is in~$A_i$ if the best search point on the island is in~$A_i$. Let $s_i$ be a lower bound for the probability that in one generation a fixed island in $A_i$ finds a search point in $A_{i+1} \cup \dots \cup A_m$.

Further assume that for each edge in the migration topology in every iteration there is a probability of at least $\pinform$ that the following holds, independently from other edges and for all~$1 \le i < m$. If the source island is in $A_i$ then after the generation the target island is in~$A_i \cup \dots \cup A_m$.
Then the expected parallel running time of the island model is bounded by
\begin{enumerate}
\item $\Oh{\frac{1}{\pinform^{1/2}} \sum_{i=1}^{m-1} \frac{1}{s_i^{1/2}}} + \frac{1}{\mu} \sum_{i=1}^{m-1} \frac{1}{s_i}$
    for every ring graph or any other strongly connected\footnote{A directed graph is strongly connected if for each pair of vertices $u, v$ there is a directed path from~$u$ to~$v$ and vice versa.} topology,
\item $\Oh{\frac{1}{\pinform^{2/3}} \sum_{i=1}^{m-1} \frac{1}{s_i^{1/3}}} + \frac{1}{\mu} \sum_{i=1}^{m-1} \frac{1}{s_i}$ for every undirected grid or torus graph with side lengths at least $\sqrt{\mu} \times \sqrt{\mu}$,
\item $\Oh{\frac{m \log(\mu) + \sum_{i=1}^{m-1} \log(1/s_i)}{\pinform}} + \frac{1}{\mu} \sum_{i=1}^{m-1} \frac{1}{s_i}$ for the $(\log \mu)$-dimensional hypercube graph,
\item $\Oh{m/\pinform} +\frac{1}{\mu} \sum_{i=1}^{m-1} \frac{1}{s_i}$ for the complete topology $K_\mu$, as well as \newline$\Oh{m + \frac{m \log \mu}{\min\{\pinform \mu, 1\}}} +\frac{1}{\mu} \sum_{i=1}^{m-1} \frac{1}{s_i}$.
\end{enumerate}
\end{theorem}

A remarkable feature of our method is that it can automatically transfer upper bounds for panmictic EAs to parallel versions thereof. The only requirement is that bounds on panmictic EAs have been derived using the fitness-level method, and that the partition $A_1, \dots, A_m$ and the probabilities for improvements $s_1, \dots, s_{m-1}$ used therein are known. Then the expected parallel time of the corresponding island model can be estimated for all mentioned topologies simply by plugging the $s_i$ into Theorem~\ref{the:generalBounds}. Fortunately, many published runtime analyses use the fitness-level method---either explicitly or implicitly---and the mentioned details are often stated or easy to derive. Hence even researchers with limited expertise in runtime analysis can easily reuse previous analyses to study parallel EAs.

Further note that we can easily determine which choice of $\mu$, the number of islands, will give an upper bound of order~$1/\mu \cdot \sum_{i=1}^{m-1} 1/s_i$---the best upper bound we can hope for, using the fitness-level method. In all bounds from Theorem~\ref{the:generalBounds} we have a first term that varies with the topology and $\pinform$, and a second term that is always $1/\mu \cdot \sum_{i=1}^{m-1} 1/s_i$. The first term reflects how quickly information about good fitness levels is spread throughout the island model. Choosing $\mu$ such that the second term becomes asymptotically as large as the first one, or larger, we get an upper bound of $\Oh{1/\mu \cdot \sum_{i=1}^{m-1} 1/s_i}$. For settings where $\sum_{i=1}^{m-1} 1/s_i$ is an asymptotically tight upper bound for a single island, this corresponds to an asymptotic linear speedup. The maximum feasible value for $\mu$ depends on the problem, the topology and the transmission probability~$\pinform$.

\begin{table}[hp]
\centering
\begin{tabular}{|p{\baselineskip}l@{\;}|@{\;}l@{\;}|@{\;}l@{\;}|@{\;}l@{\;}|@{\;}l@{\;}|@{\;}l@{\;}|}\hline
&                     & \EA               & Ring            & Grid/Torus            & Hypercube   & Complete\\\hline
\multirow{4}{\baselineskip}{\begin{sideways}\centering{}$\OM$\hspace*{0.0cm}\end{sideways}}
&best $\mu$          &                   & $\mu=\Theta(\log n)$               & $\mu=\Theta(\log n)$    & $\mu=\Theta(\log n)$ & $\mu=\Theta( \log n)$\\
&$\E{\paralleltime}$ & $\Theta(n\log n)$ & $\Oh{n}$                  & $\Oh{n}$       & $\Oh{n}$            & $\Oh{n}$\\
&$\E{\sequentialtime}$ & $\Theta(n\log n)$ & $\Oh{n \log n}$                   & $\Oh{n \log n}$       & $\Oh{n \log n}$            & $\Oh{n \log n}$\\
&$\E{\communication}$ & $0$ & $\Oh{ n \log n}$                   & $\Oh{ n \log n}$       & $\Oh{ n (\log n) \log \log n}$            & $\Oh{ n \log^2 n}$\\\hline
\multirow{4}{\baselineskip}{\begin{sideways}\centering{}$\LO$\hspace*{0.0cm}\end{sideways}}
&best $\mu$          &                   & $\mu=\Theta(n^{1/2})$                & $\mu=\Theta(n^{2/3})$        & $\mu=\Theta\left( \frac{n}{\log n}\right)$ & $\mu=\Theta( n)$\\
&$\E{\paralleltime}$ & $\Theta(n^2)$     &  $\Oh{n^{3/2}}$             & $\Oh{n^{4/3}}$      & $\Oh{n \log n}$       & $\Oh{n}$\\
&$\E{\sequentialtime}$ & $\Theta(n^2)$     &  $\Oh{n^2}$             & $\Oh{n^2}$      & $\Oh{n^2}$       & $\Oh{n^2}$\\
&$\E{\communication}$ & $0$     &  $\Oh{ n^2}$             & $\Oh{ n^2}$      & {$\Oh{ n^2 \log^2 n}$}       & $\Oh{ n^3}$\\\hline
\multirow{4}{\baselineskip}{\begin{sideways}\centering{}unimodal\hspace*{0.0cm}\end{sideways}}
&best $\mu$          &                   & $\mu=\Theta(n^{1/2})$            & $\mu=\Theta(n^{2/3})$        & $\mu=\Theta\left( \frac{n}{\log n}\right)$ & $\mu=\Theta( n)$\\
&$\E{\paralleltime}$ & $O(dn)$     &  $\Oh{d n^{1/2}}$             & $\Oh{dn^{1/3}}$     & $\Oh{d \log n}$        & $\Oh{d}$\\
&$\E{\sequentialtime}$ & $O(dn)$     &  $\Oh{dn}$             & $\Oh{dn}$     & $\Oh{dn}$        & $\Oh{dn}$\\
&$\E{\communication}$ & $0$     &  $\Oh{ dn}$             & $\Oh{ dn}$     & {$\Oh{ dn\log n}$}        & $\Oh{ dn^2}$\\\hline
\multirow{4}{\baselineskip}{\begin{sideways}\centering{}$\Jump_k$\hspace*{0.0cm}\end{sideways}}
&best $\mu$          &                   & $\mu=\Theta(n^{k/2})$               & $\mu=\Theta(n^{2k/3})$     & $\mu=\Theta( n^{k-1})$ & $\mu=\Theta( n^{k-1})$\\
&$\E{\paralleltime}$ & $\Theta(n^k)$     & $\Oh{n^{k/2}}$             & $\Oh{n^{k/3}}$    & $\Oh{n}$       & $\Oh{n}$\\
&$\E{\sequentialtime}$ & $\Theta(n^k)$     & $\Oh{n^k}$             & $\Oh{n^k}$    & $\Oh{n^k}$       & $\Oh{n^k}$\\
&$\E{\communication}$ & $0$     & $\Oh{ n^k}$             & $\Oh{ n^k}$    & {$\Oh{ kn^k \log n}$}       & $\Oh{ n^{2k-1}}$\\
\hline
\end{tabular}\smallskip
\caption{Asymptotic bounds on expected parallel ($\paralleltime$, number of generations) and sequential ($\sequentialtime$, number of function evaluations) running times and expected communication efforts ($\communication$, total number of migrated individuals) for various $n$-bit functions and island models with $\mu$ islands running the \EA and using migration probability~$\pinform=1$. The number of islands $\mu$ was always chosen to give the best possible upper bound on the parallel running time, while not increasing the upper bound on the sequential running time by more than a constant factor. For unimodal functions $d+1$ denotes the number of function values. See~\cite{Droste2002} for bounds for the \EA. Results for $\Jump_k$ were restricted to $3 \le k = O(n/\log n)$ for simplicity. All upper bounds for \OM and \LO stated here are asymptotically tight, as follows from general results in~\cite{Sudholt2012c}.}
\label{tab:running-times-best-mu-p=1}
\end{table}

We give simple examples that demonstrate how our method can be applied. Our examples are from pseudo-Boolean optimization, but the method works in any setting where the fitness-level method is applicable. The simple \EA is used on each island (see Section~\ref{sec:Preliminaries} for details).
Table~\ref{tab:running-times-best-mu-p=1} summarizes the resulting running time bounds for the considered algorithms and problem classes. For simplicity we assume $\pinform=1$; a more detailed table for general transmission probabilities is presented in the appendix, see Table~\ref{tab:running-times-full-throttle}. The number of islands $\mu$ was chosen as explained above: to give the smallest possible parallel running time, while not increasing the sequential time, asymptotically. The table also shows the expected \emph{communication effort}, defined as the total number of individuals migrated throughout the run. This quantity is proportional to the parallel expected running time, with a factor depending on the number of islands and the topology. Details are given in Theorems~\ref{the:method-ring},~\ref{the:method-torus},~\ref{the:method-hypercube},~\ref{the:method-completegraph}, and~\ref{the:method-completegraph-refined}.
The functions used in this table are explained in Section~\ref{sec:Preliminaries}. Table~\ref{tab:running-times-full-throttle} in the appendix shows all our results for a variable number of islands~$\mu$ and variable transmission probabilities~$\pinform$.

The method has already found a number of applications and it spawned a number of follow-up papers. After the preliminary version of this work~\cite{Lassig2010a} was presented, the authors applied it for various problems from combinatorial optimization: the sorting problem (as maximizing sortedness), finding shortest paths in graphs, and Eulerian cycles~\cite{Lassig2011a}. Very recently, Mambrini, Sudholt, and Yao~\cite{Mambrini2012} also used it for studying how quickly island models find good approximations for the NP-hard \textsc{SetCover} problem. This work has also led to the discovery of simple adaptive schemes for changing the number of islands dynamically throughout the run, see L{\"a}ssig and Sudholt~\cite{Lassig2011}. These schemes lead to near-optimal parallel running times, while asymptotically not increasing the sequential running time on many examples~\cite{Lassig2011}. These schemes are tailored towards island models with complete topologies, which includes offspring populations as special case. The study of offspring populations in comma strategies is another recent development that was inspired by this work~\cite{Rowe2012}.

\section{Preliminaries}
\label{sec:Preliminaries}

In our example applications we consider the maximization of a pseudo-Boolean function $f \colon \{0, 1\}^n \to \R$. It is easy to adapt the method for minimization. The number of bits is always denoted by~$n$.
The following well known example functions have been chosen because they exhibit different probabilities for finding improvements in a typical run of an EA.
For a search point $x \in \{0, 1\}^n$ write $x = x_1 \dots x_n$, then $\OM(x) := \sum_{i=1}^n x_i$ counts the number of ones in $x$ and $\LO(x) := \sum_{i=1}^n \prod_{j=1}^i x_i$ counts the number of leading ones in $x$, i.\,e., the length of the longest prefix containing only 1-bits. A function is called \emph{unimodal} if every non-optimal search point has a Hamming neighbor (\ie, a point with Hamming distance 1 to it) with strictly larger fitness. Observe that \LO{} is unimodal as flipping the first 0-bit results in a fitness increase. For \LO{} every non-optimal point has exactly one Hamming neighbor with a better fitness.
For $1 \le k \le n$ we also consider
\[
\Jump_k := \begin{cases}
k + \sum_{i=1}^n x_i, & \textrm{if $\sum_{i=1}^n x_i \le n-k$ or $x = 1^n$,}\\
\sum_{i=1}^n (1-x_i) & \textrm{otherwise}.
\end{cases}
\]
This function has been introduced by Droste, Jansen, and Wegener~\cite{Droste2002} as a function with tunable difficulty as evolutionary algorithms typically have to perform a jump to overcome a gap by flipping $k$ specific bits. It is also interesting because it is one of very few examples where crossover has been proven to be essential~\cite{Jansen2002,Koetzing2011a}.

We are interested in the following performance measures. First we define the \emph{parallel running time} $\paralleltime$ as the number of generations until the first global optimum is evaluated. The \emph{sequential running time} $\sequentialtime$ is defined as the number of function evaluations until the first global optimum is evaluated. It thus captures the overall effort across all processors. In both measures we allow ourselves to neglect the cost of the initialization as this only adds a fixed term to the running times.

The \emph{speedup} is defined as the ratio of the expected running time of a single island and the expected running time of a parallel EA with $\mu$ islands. This corresponds to the notion of a \emph{weak orthodox speedup} in Alba's taxonomy~\cite{Alba2002}. If the speedup is at least of order $\mu$, i.\,e., if it is $\Omega(\mu)$, we speak of a \emph{linear speedup}. In this work it is generally understood in an asymptotic sense, unless we call it a \emph{perfect linear speedup}.

We also define the \emph{communication effort} $\communication$ as the total number of individuals migrated to other islands during the course of a run. Depending on the parallel architecture, communication between processors can be expensive in terms of time and bandwidth used. Therefore, this measure can be an important factor for determining the performance of a parallel EA.

Our method for proving upper bounds is based on the fitness-level method~\cite{Wegener2002,Droste2002}.
The idea is to partition the search space into sets $A_1, \dots, A_m$ called \emph{fitness levels} that are ordered with respect to fitness values. We say that an algorithm is in $A_i$ or on level~$i$ if the current best individual in the population is in $A_i$. An evolutionary algorithm where the best fitness value in the population can never decrease (called an \emph{elitist} EA) can only improve the current fitness level. If one can derive lower bounds on the probability of leaving a specific fitness level towards higher levels, this yields an upper bound on the expected running time.
\begin{theorem}[Fitness-level method]
\label{the:fitness-level-method}
For two sets $A,B\subseteq \{0,1\}^n$ and a fitness function $f$ let $A <_f B$ if $f(a)<f(b)$ for all $a\in A$ and all $b \in B$. Partition the search space into non-empty sets $A_1, A_2, \dots, A_m$ such that
$
A_1<_f A_2 <_f \dots <_f A_m
$
and $A_m$ only contains global optima.
For an elitist EA let $s_i$ be a lower bound on the probability of creating a new offspring in $A_{i+1}\cup \cdots \cup A_m$, provided the population contains a search point in $A_i$. Then the expected number of iterations of the algorithm to find the optimum is bounded by
\[
\sum_{i=1}^{m-1} \frac{1}{s_i}\;.
\]
\end{theorem}
The fitness-level method has also been applied to other elitist optimization methods, including elitist ant colony optimizers~\cite{Gutjahr2008a,Neumann2009} and a binary particle swarm optimizer~\cite{Sudholtsubmitteda}. It gives rise to powerful tail inequalities~\cite{Zhou2012} and it can be used to prove lower bounds as well, when combined with additional knowledge on transition probabilities~\cite{Sudholt2012c}. Finally, Lehre~\cite{Lehre2011} recently showed that the fitness-level method can be extended towards non-elitist EAs with additional mild conditions on transition probabilities and the population size.

In the following we apply the fitness-level method to parallel EAs.
For the considered EAs we assume that there is a \emph{migration topology}, given by a directed graph. Islands represent vertices of the topology and directed edges indicate neighborhoods between the islands.
We often describe undirected graphs for use as migration topology, understanding that for an undirected edge $\{u, v\}$ we have two directed edges $(u, v)$ and $(v, u)$. In other words, though formally the migration topology is a directed graph, we often use the language of undirected graphs to describe it.

Our methods for proving upper bounds require that the islands run elitist evolutionary algorithms. All islands create new offspring independently by mutation and/or recombination among individuals in the island.
In every generation there is a chance that migration will send an individual on the current best fitness level to some target island, and that this individual will be included on the target island. This would effectively increase the fitness level of the target island to the current best level (or an even better one). For every pair of connected islands, we call this probability \emph{transmission probability} and denote it~$\pinform$. Note that for any pair of islands, the mentioned transmission events are independent.

The transmission probability can model various settings, where randomness and stochasticity may be involved:
\begin{itemize}
\item migrations do not take place in every generation, but only probabilistically with probability~$\pinform$,
\item islands do not automatically select individuals on the best fitness level for emigration, but there is a probability of at least~$\pinform$ that this happens,
\item similarly, islands do not automatically include immigrants on higher fitness levels, but only with probability at least~$\pinform$,
\item during migration crossover is performed, and $\pinform$ is a lower bound on the probability that crossover does not disrupt the fitness of an individual on a current best fitness level (if a crossover probability $p_c$ is used, then clearly $\pinform \ge 1-p_c$),
\item the physical architecture suffers from transient faults and $\pinform$ is a lower bound on the probability that migration is executed correctly.
\end{itemize}
Of course, the transmission probability can also model any combination of the above, in which case the product of all above probabilities gives a lower bound on the transmission probability.

Most of our results also apply when instead of probabilistic migration a fixed migration interval~$\migint$ is used. This is similar to a migration probability~$\pinform=1/\migint$; in fact, it can be regarded as a derandomized or quasi-random version of probabilistic migration. With a fixed migration interval the variance in the information propagation is reduced, and all islands operate in synchronicity. Probabilistic migrations are asynchronous; this simplifies the analysis as we do not need to keep track on how much time has passed since the last migration. We expect our results for probabilistic migration to transfer to the study of migration intervals. The only notable exception is the case of a complete topology, when the migration probability is rather small (Theorem~\ref{the:method-completegraph-refined}) as there synchronous and asynchronous migrations lead to different effects.

As elaborated above, our method is robust and it applies in various settings, and for various types of EAs simulated on the islands.
In our applications for illustrating concrete speedups for test problems, we use a simple \EA for all islands. The \EA maintains a single current search point, and in each generation it creates an offspring by mutation. The offspring replaces its parent if its fitness is not worse. The resulting island model is shown in Algorithm~\ref{alg:Evolutionary}.

\begin{algorithm}[h]
    \caption{Parallel \EA with $\mu$ islands and migration probability $\pinform$}
    \begin{tabbing}
     \qquad \= \small \qquad \=  \qquad \= \kill
     \textbf{For all} $1 \le i \le \mu$ choose $x^i \in \{0, 1\}^n$ uniformly at random.\\
     \REPEAT\\
    \> \textbf{For all} $1 \le i \le \mu$ \textbf{do in parallel}\\
    \> \>  Create $y^i$ by flipping each bit in $x^i$ with probability $1/n$.\\
    \> \>  \IF $f(y^i) \ge f(x^i)$ \THEN $x^i := y^i$.\\
    \> \>  Send a copy of $x^i$ to each neighboring island, independently with prob.~$\pinform$.\\
    \> \>  Choose $z^i$ with maximum fitness among all incoming migrants.\\
    \> \>  \IF $f(z^i) \ge f(x^i)$ \THEN $x^i := z^i$.
    \end{tabbing}
    \label{alg:Evolutionary}
\end{algorithm}

\section{Proving Upper Bounds for Parallel EAs}
\label{sec:general-upper-bounds}

\subsection{A General Upper Bound}

Now we describe how to prove upper bounds on the running time of parallel EAs.
In contrast to panmictic EAs, in an island model several islands might participate in the search for improvements from the current-best fitness level. The number of islands may vary over time according to the spread of information.

The following theorem transfers upper bounds for panmictic EAs derived by the fitness-level method into upper bounds for parallel EAs in a systematic way. \begin{theorem}[Fitness-level method for parallel EAs]
\label{the:fitness-levels-for-parallel-EAs}
Consider a partition of the search space into fitness levels $A_1 <_f A_2 <_f \dots <_f A_m$ such that $A_m$ only contains global optima.
Let $s_i$ be (a lower bound on) the probability that a fixed island running an elitist EA creates a new offspring in $A_{i+1}\cup \cdots \cup A_m$, provided the island contains a search point in $A_i$.
Let $\mu_t$ for $t \in \N$ denote (a lower bound on) the number of islands that have discovered an individual in $A_{i} \cup \cdots \cup A_m$ in the $t$-th generation after the first island has found such an individual.
Then the expected parallel running time of the parallel EA on~$f$ is bounded by
\[
\E{\paralleltime} \;\le\; \sum_{i=1}^{m-1} \sum_{t=0}^\infty (1-s_i)^{\sum_{j=1}^{t}\mu_j}\;.
\]
\end{theorem}
\begin{proof}
Let $T_i$ denote the random time until the first island finds an individual on a fitness level $i+1, \dots, m$, starting with at least one individual on fitness level~$i$ in the whole population.
The expected parallel running time can be written as
\[
E(\paralleltime) = \sum_{i=1}^{m-1} \E{T_i}
                    = \sum_{i=1}^{m-1} \sum_{t=1}^\infty \Prob{T_i \geq t}\\
                   =  \sum_{i=1}^{m-1} \sum_{t=0}^\infty \Prob{T_i \ge t+1}.
\]
A necessary condition for $T_i \ge t+1$ is that during all $t$ generations after the first individual has reached fitness level~$i$ all islands are unsuccessful in finding an improvement. In the $j$-th of these generations there are at least $\mu_j$ islands, each being successful with probability at least $s_i$. Using that the islands create new offspring independently, the probability of all islands being unsuccessful is at most $(1-s_i)^{\mu_j}$. Thus,
\[
\sum_{i=1}^{m-1} \sum_{t=0}^\infty \Prob{T_i \ge t+1}
                   \le \sum_{i=1}^{m-1} \sum_{t=0}^\infty \prod_{j=1}^{t} (1-s_i)^{\mu_j}\\
                   = \sum_{i=1}^{m-1} \sum_{t=0}^\infty (1-s_i)^{\sum_{j=1}^{t}\mu_j}\;.\qedhere
\]
\end{proof}

The upper bound from Theorem~\ref{the:fitness-levels-for-parallel-EAs} is very general as it does not restrict the communication among the islands in any way. These aspects are hidden in the definition of the variables $\mu_t$. When looking at one particular fitness level, say level~$i$, we also speak of islands being \emph{informed} if and only if they contain an individual on level~$i$. The variable $\mu_t$ then gives the number of informed islands $t$ generations after the first island has been informed.

The spread of information obviously depends on the migration topology, the migration interval, and the selection strategies used to choose migrants that are sent and how migrants are included in the population.
The basic method works for all choices of these design aspects. We elaborate on these aspects and then move on to more specific scenarios where we can obtain more concrete results.

\subsection{How to Deal with Migration Intervals}

With a migration interval of $\migint > 1$ the $\mu_t$-value remains fixed for periods of $\migint$ generations. For appropriate~$t$ then $\mu_{t} = \mu_{t+1} = \dots = \mu_{t+\migint-1}$.
As the $\mu_t$-values are non-decreasing with $t$, the sum of $\mu$-values is at least $\sum_{j=1}^t \mu_j \ge \migint \sum_{j=1}^{t/\migint} \mu_{(j-1)\migint+1}$.
This implies the following simplified upper bound.
\begin{corollary}
For a parallel EA with migration interval $\migint$ the bound from Theorem~\ref{the:fitness-levels-for-parallel-EAs} simplifies to
\[
   \E{\paralleltime} \le           \sum_{i=1}^{m-1} \sum_{t=0}^\infty (1-s_i)^{\migint \sum_{j=1}^{t/\migint}\mu_{(j-1)\migint+1}}\;.
\]
\end{corollary}
The values $\mu_{(j-1)\migint}$ can be estimated like the values $\mu_{j}$ in a setting with ${\migint = 1}$.
In order to keep the presentation simple, in the following applications we only consider the case that $\migint = 1$, \ie, migration happens in every generation. This reflects common principles used in fine-grained or cellular evolutionary algorithms.
The following considerations can always be combined with the above arguments to handle migration intervals larger than 1.

\subsection{Stochastic Communication and Finding Improvements}

In order to arrive at more concrete bounds on the parallel running time for common migration topologies, we need to understand how the number of informed islands grows on each fitness level, i.\,e., the growth curves underlying the $\mu_j$-variables. Note that these variables are random variables in all settings where we have a transmission probability less than~1. This means that getting a closed formula for the expected parallel running time is not easy. In Theorem~\ref{the:fitness-levels-for-parallel-EAs} we cannot simply replace the $\mu_j$-variables by their expectations as by Jensen's inequality this would yield an estimation in the wrong direction (i.\,e., it would give a lower bound where an upper bound is needed). More work is required in order to arrive at closed formulas for common topologies.

Instead of arguing with the random number of informed islands, it is easier to argue with expected hitting times for the time until a specified number of islands is informed. If we know such expected hitting times, or upper bounds thereof, we can estimate the time until the parallel EA finds a better fitness level.
\begin{lemma}
\label{lem:fitness-level-k}
Consider an island model running elitists EAs and fix some fitness level~$i$ with success probability~$s_i$ for each island. Let $\Tinform{k}$ denote the random number of generations until at least $k$ islands are informed.
Then for every $k \le \mu$ the expected time until this fitness level is left towards a better one is at most
\[
\E{\Tinform{k}} + 1 + \frac{1}{k} \cdot \frac{1}{s_i}.
\]
\end{lemma}
\begin{proof}
After $\E{\Tinform{k}}$ expected generations there are at least $k$ informed islands. Then the probability of leaving the fitness level is at least
$1-(1-s_i)^k$ and the expected time is bounded by
\begin{equation}
\label{eq:Rowe}
\frac{1}{1-(1-s_i)^k} \le 1 + \frac{1}{k} \cdot \frac{1}{s_i},
\end{equation}
where the inequality is due to Jon Rowe~\cite[Lemma~3]{Rowe2012}, stated as Lemma~\ref{lemma:parallel-bounds} in the appendix. Together, this proves the claim.
\end{proof}
A good choice for $k$ is one where $\E{\Tinform{k}} \approx \frac{1}{k} \cdot \frac{1}{s_i}$ as this is likely to minimize the bound from Lemma~\ref{lem:fitness-level-k}, at least asymptotically.

The lemma ignores the fact that during the first $\Tinform(k)$ generations islands can already find improvements. It also ignores that the number of islands might grow beyond~$k$ after this time. However, we will see that for appropriate choices of~$k$, the lemma still gives near-optimal results. In the first generations the number of islands is likely to be too small anyway to yield a significant benefit. In addition, after $k$ islands have been informed this number is large enough to guarantee that improvements are found quickly, for appropriate~$k$.

\subsection{Information Propagation in Networks}

It remains to estimate the first hitting time for informing a certain number of vertices. Note that this is similar to studying growth curves and takeover times. In fact, $\xi(\mu)$ is the expected time until the whole island model is informed. Growth curves and takeover times have been studied in artificial settings where no variation takes place, see~\cite{Rudolph2000a,Rudolph2006,Sarma1997,Alba2004,Giacobini2003,Giacobini2003a,Giacobini2005a,Giacobini2005}
or recent surveys~\cite[Chapter~4]{Luque2011}, \cite{Sudholt2012a}.

In the following, we refer to our model of transmission probabilities as it is a general model that captures many stochastic components in the dynamic behavior of island models. But at the same time it is simple enough to allow for a theoretical analysis.

Transmission probabilities give rise to a stochastic information propagation process in networks. Each informed vertex in the network independently tries to inform all its neighbors in every iteration, and information is successfully transmitted across any of these edges with probability~$p$. This process was studied by Rowe, Mitavskiy, and Cannings~\cite{Rowe2008}, who considered the \emph{propagation time} as the time until all vertices in the network are informed.
They presented bounds for interesting graph classes as well as a general upper bound of
\[
\frac{8\diam(G) + 8\log n}{p(1-e^{-1})}
\]
for the propagation time on an undirected graph~$G$. Thereby $\diam(G)$ denotes the \emph{diameter} of~$G$, defined as the maximum number of edges on any shortest path between two vertices in the graph.

Interestingly, the same probabilistic process also underlies the way randomized search heuristics find shortest paths in weighted undirected graphs. Doerr, Happ, and Klein~\cite{Doerr2007e,Doerr2011d} showed that the \EA can find shortest paths in graphs by simulating the Bellman-Ford algorithm. The task is to find shortest paths from a source~$v^*$ to all other vertices. For vertices whose shortest paths have few edges, shortest paths are found quickly. In our language these vertices would be called \emph{informed}. If $u$ is informed and the graph contains an edge $\{u, v\}$, then $v$ can become informed with a fixed probability during a lucky mutation, if the shortest path from $v^*$ to~$v$ contains~$u$. This way, shortest paths propagate through the graph in the same fashion as information does. The same can be observed for ant colony optimizers~\cite{Sudholt2011a}.

Doerr, Happ, and Klein~\cite{Doerr2007e} independently used a different argument for bounding the expected propagation time. Fix a shortest path in the graph, leading from~$v^*$ to some fixed vertex~$v$. In every generation there is a chance of informing the first uninformed vertex on the path, until eventually the information reaches~$v$. If the path has at least $\log n$ edges, the time until~$v$ is informed is highly concentrated. Using tail bounds, the probability of significantly exceeding the expectation is very small. This allows us to apply a union bound for all considered vertices~$v$.

Following the proof of~\cite[Lemma~3]{Doerr2007e}, we get the following lemma. An advantage over the general bound from~\cite{Rowe2008} is that it not only bounds the propagation time for the whole network. It also bounds expected hitting times for informing smaller numbers of vertices.
\begin{lemma}
\label{lem:propagation-Doerr}
Consider propagation with transmission probability~$\pinform$ on any undirected graph where initially a single vertex~$v^*$ is informed. For $i \in \N_0$ let $V_i$ contain all vertices~$v$ whose shortest path from~$v^*$ to~$v$ contains $i$ edges. Let $s_k := \sum_{i=1}^k |V_i|$.
The probability of not having informed $s_k$ vertices in time $\lambda k/\pinform$, $\lambda \ge 2$,
is at most
\[
s_k \cdot \exp\left(-\frac{(\lambda-1)^2}{2\lambda} \cdot k\right)
\le
s_k \cdot \exp\left(-\frac{\lambda k}{8}\right).
\]
The expected time until $s_k$ vertices are informed is at most
\[
\frac{\frac{c}{c-1}  \cdot \max\left\{4k, 8\ln(c s_k)\right\}}{\pinform}
\]
for every $c > 1$.
\end{lemma}
\begin{proof}
The first claim follows from the proof of Lemma~3 in~\cite{Doerr2007e} and the fact that $(\lambda-1)^2/\lambda \ge \lambda/4$ for $\lambda \ge 2$.

If $8/k \cdot \ln(cs_k) \ge 2$ we use $\lambda := 8/k \cdot \ln(cs_k)$ and have that after $\lambda k/\pinform$ iterations the probability of not having informed all vertices is at most
\[
s_k \cdot \exp\left(-\ln(c s_k)\right) = \frac{1}{c}.
\]
If not, we repeat the argument with another phase of $\lambda k/\pinform$ iterations. As each phase is successful with probability at least $1-1/c$, the expected propagation time is at most
\[
\frac{1}{1-1/c} \cdot \frac{\lambda k}{\pinform} = \frac{\frac{c}{c-1} \cdot 8\ln(c s_k)}{\pinform}.
\]
If $8/k \cdot \ln(cs_k) < 2$ then $k/4 > \ln(c s_k)$. The first statement with $\lambda := 2$ then gives a probability bound of
\[
s_k \cdot \exp\left(-\frac{k}{4}\right) \le s_k \cdot \exp\left(-\ln(c s_k)\right) \le \frac{1}{c}
\]
and using the same arguments as before we get a time bound of
\[
\frac{1}{1-1/c} \cdot \frac{\lambda k}{\pinform} = \frac{\frac{c}{c-1} \cdot 2k}{\pinform}.
\]
\end{proof}
Note that putting $k := \diam(G)$ and $c=2$, we get a bound of
\[
\max\left\{\frac{4\diam(G)}{\pinform}, \frac{16\ln(2n)}{\pinform}\right\}
\le \frac{4\diam(G) + 11.2 \log(n) + 11.2}{\pinform}.
\]
For all non-empty graphs this is better than the general upper bound
\[
\frac{8\diam(G) + 8\log n}{\pinform(1-e^{-1})} \approx \frac{12.7\diam(G) + 12.7\log n}{\pinform}
\]
from~Rowe, Mitavskiy, and Cannings~\cite{Rowe2008}. However, the asymptotic behavior of both bounds is the same as $(x+y)/2 \le \max\{x, y\} \le x+y$ for all $x, y \in \R^+_0$, hence $\max\{x, y\} = \Theta(x + y)$.

Now we are prepared to analyze parallel EAs with concrete topologies.

\section{Parallel EAs with Ring Structures}
\label{sec:ring}

We start with ring graphs as they are often used as topologies~\cite{Tomassini2005}. Rings can either be unidirectional, in which case there is exactly one directed cycle, or bidirectional, when all edges are undirected. The following theorem holds for both kinds of graphs, and in fact for all strongly connected graphs. Recall that a directed graph is called \emph{strongly connected} if for every two vertices $u, v$ there is a directed from $u$ to~$v$ (implying that there is also a path from~$v$ to~$u$).
\begin{theorem}
\label{the:method-ring}
Consider an island model running elitists EAs on a function~$f$ with a fitness-level partition $A_1 <_f \dots <_f A_m$ and success probabilities $s_1, \dots, s_{m-1}$. Let $\pinform$ be (a lower bound on) the probability that a specific island on fitness level~$i$ informs a specific neighbor in the topology in one generation.
The expected parallel running time on~$f$ with an unidirectional or bidirectional ring and $\mu$ islands---or in fact any strongly connected topology---is bounded by
\[
\frac{2}{\pinform^{1/2}} \sum_{i=1}^{m-1} \frac{1}{s_i^{1/2}} + \frac{1}{\mu} \cdot \sum_{i=1}^{m-1} \frac{1}{s_i}\;.
\]
The expected communication effort for ring graphs is by a factor of at most $2\pcomm \mu$ larger than the expected parallel time.
\end{theorem}
The shape of this formula deserves some explanation. The second term $ \frac{1}{\mu} \cdot \sum_{i=1}^{m-1} \frac{1}{s_i}$ is by a factor of $\mu$ smaller than the upper bound for a single island by Theorem~\ref{the:fitness-level-method}. If the latter is asymptotically tight, the second term in Theorem~\ref{the:method-ring}, regarded in isolation, would give a perfect linear speedup. The first term is related to the speed at which information is propagated through the island model. Unlike for the second term, it is independent of~$\mu$, but it depends on the transmission probability~$\pinform$. We do have a linear speedup if the first term $\frac{2}{\pinform^{1/2}} \sum_{i=1}^{m-1} \frac{1}{s_i^{1/2}}$ asymptotically does not grow faster than the second term, again assuming that the bound for a single island is tight.

As $\mu$ grows, the second term becomes smaller, while the first term remains fixed. So if we have a linear speedup for small~$\mu$, there is a point where with growing~$\mu$ the linear speedup disappears. This threshold can be easily computed by checking which value of~$\mu$ gives rise to the first and second terms being of equal asymptotic order. As will be seen in the next sections, the same also holds for other migration topologies.
\begin{proof}[Proof of Theorem~\ref{the:method-ring}]
For the unidirectional ring we have $\E{\Tinform{k}} \le (k-1)/\pinform$ since a new island is informed with probability at least $\pinform$. As this happens independently in each generation, the expected waiting time until this happens is at most~$1/\pinform$.
In fact, this argument holds for all strongly connected topologies and in particular for the bidirectional ring.

Now, if $1 \le k := \pinform^{1/2}/s_i^{1/2} \le \mu$ (ignoring rounding issues), by Lemma~\ref{lem:fitness-level-k} the expected number of generations on fitness level~$i$ is bounded by
\[
\frac{k-1}{\pinform} + 1 + \frac{1}{k} \cdot \frac{1}{s_i} \le \frac{1}{\pinform^{1/2}s_i^{1/2}} + \frac{1}{\pinform^{1/2}s_i^{1/2}} = \frac{2}{\pinform^{1/2}s_i^{1/2}}.
\]
In case $\pinform^{1/2}/s_i^{1/2} < 1$ we trivially get an upper bound of
\[
\frac{1}{s_i} \le \frac{1}{\pinform^{1/2}s_i^{1/2}}.
\]
If $\pinform^{1/2}/s_i^{1/2} > \mu$, Lemma~\ref{lem:fitness-level-k} for $k := \mu$ gives an upper bound of
\[
\frac{\mu-1}{\pinform} + 1 + \frac{1}{\mu} \cdot \frac{1}{s_i} > \frac{1}{\pinform^{1/2}s_i^{1/2}} + \frac{1}{\mu} \cdot \frac{1}{s_i}.
\]
Taking the maximum of the above upper bounds gives
\[
\max\left(\frac{2}{\pinform^{1/2} s_i^{1/2}}, \frac{1}{\pinform^{1/2} s_i^{1/2}} + \frac{1}{\mu} \cdot \frac{1}{s_i}\right)
\le
\frac{2}{\pinform^{1/2} s_i^{1/2}} + \frac{1}{\mu} \cdot \sum_{i=1}^{m-1} \frac{1}{s_i}\;.
\]
Summing over all fitness levels proves the claim.
\end{proof}
As remarked in the proof, the bound from Theorem~\ref{the:method-ring} holds for arbitrary strongly connected topologies as the unidirectional ring is a worst case for the $\mu_t$-values.

For bidirectional rings we have $\E{\Tinform{k}} \le \frac{k}{2\pinform}$ (as can be seen from applying Johannsen's drift theorem, stated in the appendix as Theorem~\ref{drift:johannsen}, to the difference to $k$ informed vertices, using $h(1)=\pinform$ and $h(x) = 2\pinform$ for $x > 1$ as drift function). This decreases the constant~2 in the first term towards $\sqrt{2}$, at the expense of an additional term~$m-1$.

Also note that if $\pinform < s_i$ then the trivial bound $1/s_i$ gives a better estimate for the time until this fitness level is left. If this holds for all fitness levels, parallelization does not give any provable speedups as information is propagated too slowly.

Contrarily, if, say, $\pinform = \Omega(1)$, compared to a single island in a ring the expected waiting time for every fitness level can be replaced by its square root. This can yield significant speedups. We make this precise for concrete functions in the following theorem. For comparing these times with runtime bounds for the \EA we refer to Table~\ref{tab:running-times-best-mu-p=1}.
\begin{theorem}
\label{the:application-ring}
The following holds for the parallel \EA with transmission probability at least $\pinform$ on a unidirectional or bidirectional ring (or any other strongly connected topology):
\begin{itemize}
\item $\E{\paralleltime} = \mathord{O}\mathord{\left(\frac{n}{\pinform^{1/2}} + \frac{n \log n}{\mu}\right)}$ for \OM{},
\item $\E{\paralleltime} = \mathord{O}\mathord{\left(\frac{d n^{1/2}}{\pinform^{1/2}} + \frac{dn}{\mu}\right)}$ for every unimodal function with $d+1$ function values,
\item $\E{\paralleltime} = \mathord{O}\mathord{\left(\frac{n^{k/2}}{\pinform^{1/2}} + \frac{n^k}{\mu}\right)}$ for $\Jump_k$ with $k \ge 2$.
\end{itemize}
\end{theorem}
\begin{proof}
For \OM{} we choose the canonical partition $A_i := \{x \mid \OM(x) = i\}$. The probability of increasing the current fitness from fitness level~$i$ is at least $s_i\geq (n-i)\cdot 1/(en)$ since there are $n-i$ Hamming neighbors of larger fitness and a specific Hamming neighbor is created with probability at least $1/n \cdot (1-1/n)^{n-1} \ge 1/(en)$.
The second sum in Theorem~\ref{the:method-ring} is
\[
\frac{1}{\mu} \cdot \sum_{i=0}^{n-1} \frac{en}{n-i} = \frac{en}{\mu} \sum_{i=1}^n \frac{1}{i} = \mathord{O}\mathord{\left(\frac{n \log n}{\mu}\right)}.
\]
The first sum in Theorem~\ref{the:method-ring} is
\begin{align*}
2\sum_{i=0}^{n-1} \left(\frac{en}{n-i} \cdot \frac{1}{\pinform}\right)^{1/2}
&\;= 2\left(\frac{en}{\pinform}\right)^{1/2} \sum_{i=1}^n \frac{1}{\sqrt{i}}\\
&\;\le 2\left(\frac{en}{\pinform}\right)^{1/2} \int_{0}^n \frac{1}{\sqrt{i}} \; \mathrm{d}i
\le  2\left(\frac{en}{\pinform}\right)^{1/2} \cdot \sqrt{n} = O(n).
\end{align*}
For unimodal functions we choose a partition $A_1, \dots, A_{d+1}$ where $A_i$ contains all search points with the $i$-th smallest function value.
The probability of improving the fitness from level $i$ is
at least $s_i \geq 1/(en)$ because there is at least one search point in the next fitness level which is at Hamming distance one.
Theorem~\ref{the:method-ring} gives an upper bound of
\[
2\sum_{i=1}^{d} \left(\frac{en}{\pinform}\right)^{1/2}  + \sum_{i=1}^{d} \frac{en}{\mu} \le 2d \cdot \left(\frac{en}{\pinform}\right)^{1/2} + \frac{den}{\mu} = \mathord{O}\mathord{\left(\frac{dn^{1/2}}{\pinform^{1/2}} + \frac{dn}{\mu}\right)}.
\]
For $\Jump_k$ functions and $i \notin \{n-k, n\}$ the fitness levels $A_i$ are chosen similarly to $\OM$, yielding the same terms in the upper bound as for $\OM$. (In fact, results are even better as the hardest fitness levels for $\OM$ are replaced by easy fitness levels.) But to reach the highest level from $n-k$ 1-bits, i.\,e., level~$A_n$, a specific bit string with Hamming distance $k$ has to be created. This has probability at least
\[
s_{n} \geq \left(\frac{1}{n}\right)^k\cdot \left(1-\frac{1}{n}\right)^{n-k}\geq \left(\frac{1}{n}\right)^k\cdot \left(1-\frac{1}{n}\right)^{n-1}\geq\frac{1}{en^k}\;.
\]
Theorem~\ref{the:method-ring} and the above bound for $\OM$ give
\[
\mathord{O}\mathord{\left(\frac{n}{\pinform^{1/2}} + \frac{n \log n}{\mu}\right)} + 2\left(\frac{en^k}{\pinform}\right)^{1/2} + \frac{en^k}{\mu} = \mathord{O}\mathord{\left(\frac{n^{k/2}}{\pinform^{1/2}} + \frac{n^k}{\mu}\right)}.\qedhere
\]
\end{proof}
The speedups obtained are indeed significant, particularly for those functions where improvements are hard to find.

The proof of Theorem~\ref{the:application-ring} uses well-known fitness-level partitions~\cite{Wegener2002,Droste2002}, and hence it simply consists of plugging in known values $s_i$ and simplifying. This shows how easy it is to obtain results for parallel EAs based on analyses of panmictic EAs.

By a strange coincidence, the speedups obtained through parallelization on ring graphs are as large as those obtained through quantum search~\cite{Johannsen2010a}.

\section{Parallel EAs with Two-Dimensional Grids and Tori}
\label{sec:torus}

For two-dimensional grids and tori we adapt Theorem~\ref{the:fitness-levels-for-parallel-EAs} in a similar manner, making an effort to get the best possible leading constant in the first term of the running time bound. We also consider applications of the resulting theorem similar to the applications for ring graphs.
\begin{theorem}
\label{the:method-torus}
Consider the setting from Theorem~\ref{the:method-ring}.
The expected parallel running time of the island model on a grid or torus topology with side lengths $\sqrt{\mu} \times \sqrt{\mu}$ is bounded by
\[
\frac{3^{5/3}}{\pinform^{2/3}} \sum_{i=1}^{m-1} \frac{1}{s_i^{1/3}} + \frac{1}{\mu} \sum_{i=1}^{m-1} \frac{1}{s_i}\;.
\]
The expected communication effort is by a factor of at most $4\pcomm \mu$ larger than the expected parallel time.
\end{theorem}
\begin{proof}
Note that within a square area of $\sqrt{k} \times \sqrt{k}$ vertices in the graph all shortest paths between any two vertices have at most $2\sqrt{k}-2$ edges.
Applying Lemma~\ref{lem:propagation-Doerr} with $k' := 2\sqrt{k}-2$, $s_{k'} \ge k$ and $c=4$ we have that for every $k \le \mu$ the expected time until $k$ islands are informed is bounded by
\[
\max\left\{\frac{8/3 \cdot \sqrt{k}-8/3}{\pinform}, \frac{32/3 \cdot \ln(4k)}{\pinform}\right\}.
\]
We also get an upper bound of $k/(2\pinform)$ using Johannsen's variable drift theorem~\cite{Johannsen2010}, Theorem~\ref{drift:johannsen} in the appendix, as follows. If there is more than one uninformed vertex, there are always at least two vertices neighboring to informed ones. So the expected number of informed vertices increases by $2\pinform$ in expectation. Applying Johannsen's drift theorem as for the bidirectional ring gives an upper bound of $k/(2\pinform)$.
It is easy to check that the best upper bound is as follows: for all $k \in \N$
\[
\min\left\{\frac{k}{2\pinform}, \max\left\{\frac{8/3 \cdot \sqrt{k}-8/3}{\pinform}, \frac{32/3 \cdot \ln(4k)}{\pinform}\right\}\right\} \le \frac{6\sqrt{k}-1}{\pinform}.
\]
Now, if $1 \le k := 3^{-2/3} \cdot (\pinform/s_i)^{2/3} \le \mu$ (ignoring rounding issues) by Lemma~\ref{lem:fitness-level-k} the expected number of generations on fitness level~$i$ is bounded by
\[
\frac{6\sqrt{k}-1}{\pinform} + 1 + \frac{1}{k} \cdot \frac{1}{s_i}
\le \frac{6 \cdot 3^{-1/3}}{\pinform^{2/3}s_i^{1/3}} + \frac{3^{2/3}}{\pinform^{2/3}s_i^{1/3}} = \frac{3^{5/3}}{\pinform^{2/3}s_i^{1/3}}.
\]
If $3^{-2/3} \cdot (\pinform/s_i)^{2/3} < 1$ we trivially get an upper bound of
\[
\frac{1}{s_i} \le \frac{3^{2/3}}{\pinform^{2/3}s_i^{1/3}}.
\]
If $3^{-2/3} \cdot \pinform^{2/3}/s_i^{2/3} > \mu$, we get for $k := \mu$ an upper bound of
\[
\frac{6\sqrt{\mu}-1}{\pinform} + 1 + \frac{1}{\mu} \cdot \frac{1}{s_i} < \frac{2 \cdot 3^{2/3}}{\pinform^{2/3}s_i^{1/3}} + \frac{1}{\mu} \cdot \frac{1}{s_i}.
\]
Taking the maximum of the above upper bounds gives
\begin{align*}
\max\left(\frac{3^{5/3}}{\pinform^{2/3}s_i^{1/3}},
\frac{2 \cdot 3^{2/3}}{\pinform^{2/3}s_i^{1/3}} + \frac{1}{\mu} \cdot \frac{1}{s_i}\right)
&\;\le \frac{3^{5/3}}{\pinform^{2/3}s_i^{1/3}} +  \frac{1}{\mu} \cdot \frac{1}{s_i}
.
\end{align*}
Summing over all fitness levels yields the claim.
\end{proof}
Note that the communication effort in one generation is asymptotically as large as for ring graphs, but for large~$\pinform$ the parallel running time is generally smaller.
If $\pinform < 3s_i$ then again the trivial upper bound $1/s_i$ is better as then the spread of information is too slow.

Compared to a single island, in a torus the expected waiting time for every fitness level can be replaced by its third root. This leads to improved upper bounds for unimodal functions and $\Jump_k$.
\begin{theorem}
The following holds for the parallel \EA with transmission probability~$\pinform$ on a grid or torus topology and side lengths at least $\sqrt{\mu} \times \sqrt{\mu}$:
\begin{itemize}
\item $\E{\paralleltime} = \mathord{O}\mathord{\left(\frac{n}{\pinform^{2/3}} + \frac{n \log n}{\mu}\right)}$ for \OM{},
\item $\E{\paralleltime} = \mathord{O}\mathord{\left(\frac{dn^{1/3}}{\pinform^{2/3}} + \frac{dn}{\mu}\right)}$ for every unimodal function with $d+1$ function values,
\item $\E{\paralleltime} = \mathord{O}\mathord{\left(\frac{n + n^{k/3}}{\pinform^{2/3}} + \frac{n^k}{\mu}\right)}$ for $\Jump_k$ with $k \ge 2$.
\end{itemize}
\end{theorem}
\begin{proof}
We choose the same partitions as in the proof of Theorem~\ref{the:application-ring}. Note that the second terms in Theorem~\ref{the:method-ring} and~\ref{the:method-torus} are identical, so we only estimate the first terms and refer to Theorem~\ref{the:application-ring} for the second terms.

For $\OM$ the first sum in Theorem~\ref{the:method-torus} is
\begin{align*}
\frac{3^{5/3}}{\pinform^{2/3}} \sum_{i=0}^{n-1} \left(\frac{en}{n-i}\right)^{1/3}
&\;= \frac{3^{5/3} e^{1/3}n^{1/3}}{\pinform^{2/3}} \sum_{i=1}^{n} \left(\frac{1}{i}\right)^{1/3}\\
&\;\le \frac{3^{5/3} e^{1/3}n^{1/3}}{\pinform^{2/3}} \int_{i=0}^{n} \left(\frac{1}{i}\right)^{1/3} \;\mathrm{d}i\\
&\;= \frac{3^{5/3} e^{1/3}n^{1/3}}{\pinform^{2/3}} \cdot \frac{3}{2} \cdot n^{2/3}
= \frac{3^{8/3}/2 \cdot e^{1/3}n}{\pinform^{2/3}}.
\end{align*}
This gives an upper bound of
\[
\mathord{O}\mathord{\left(\frac{n}{\pinform^{2/3}} + \frac{n \log n}{\mu}\right)}.
\]
For unimodal functions Theorem~\ref{the:method-torus} gives
\[
\frac{3^{5/3}}{\pinform^{2/3}} \sum_{i=1}^{d} \left(\frac{en}{\pinform}\right)^{1/3}  + \sum_{i=1}^{d} \frac{en}{\mu}
\le \frac{3^{5/3}d \cdot e^{1/3}n^{1/3}}{\pinform^{2/3}} + \frac{den}{\mu}
= \mathord{O}\mathord{\left(\frac{dn^{1/3}}{\pinform^{2/3}} + \frac{dn}{\mu}\right)}.
\]
For $\Jump_k$ we get
\[
\mathord{O}\mathord{\left(\frac{n}{\pinform^{2/3}} + \frac{n \log n}{\mu}\right)} + 3^{5/3} \cdot \frac{(en^k)^{1/3}}{\pinform^{2/3}} + \frac{en^k}{\mu} = \mathord{O}\mathord{\left(\frac{n + n^{k/3}}{\pinform^{2/3}} + \frac{n^k}{\mu}\right)}.\qedhere
\]
\end{proof}

\section{Parallel EAs with Hypercube Graphs}
\label{sec:hypercube}

Hypercube graphs are popular topologies in parallel computation. In a $d$-dimensional hypercube each vertex has a label of $d$ bits. Two vertices are neighboring if and only if their labels differ in exactly one bit. The number of vertices is then $2^d$, and each vertex has $d$ neighbors. The diameter of a $d$-dimensional hypercube is~$d$, hence only logarithmic in the size of the graph. The small diameter implies that in many communication models information is spread rapidly, even though the degree of vertices is quite small. With regard to the propagation process investigated here, we get a small first term in the following running time bound, and still have a very moderate communication effort.
\begin{theorem}
\label{the:method-hypercube}
Consider the setting from Theorem~\ref{the:method-ring}.
The expected parallel running time of the island model on a $(\log \mu)$-dimensional hypercube graph with $\mu$ islands is bounded by
\[
\frac{49m+24  \sum_{i=1}^{m-1} \log(\frac{1}{s_i}) }{\pinform}  + \frac{1}{\mu} \cdot \sum_{i=1}^{m-1} \frac{1}{s_i}.
\]
The expected communication effort is by a factor of at most $\pcomm \mu \log \mu$ larger than the expected parallel time.
\end{theorem}
\begin{proof}
In the notation of Lemma~\ref{lem:propagation-Doerr} we have for the hypercube and $1 \le k \le \log \mu$
\[
s_k = \sum_{i=1}^k \binom{\mu}{i} \ge 2^k.
\]
Invoking Theorem~\ref{lem:propagation-Doerr} with $c=2$, the expected time until $2^k$ vertices are informed is therefore at most
\[
\frac{16 \ln(2 \cdot 2^k)}{\pinform} = \frac{\frac{16}{\ln 2} \cdot (k+1)}{\pinform} < \frac{24 (k+1)}{\pinform}.
\]
By Lemma~\ref{lem:fitness-level-k} the expected time on fitness level~$i$ is hence bounded, for any integer $0 \le k \le \log \mu$, by
\begin{equation}
\label{eq:hypercube-time-on-level-i}
\frac{24(k+1)}{\pinform} + 1 + \frac{1}{2^k} \cdot \frac{1}{s_i} \le \frac{25}{\pinform} + \frac{24k}{\pinform} + \frac{1}{2^k} \cdot \frac{1}{s_i}.
\end{equation}
If $\pinform/(24s_i) < 1$, we get a trivial upper bound of~$1/s_i \le 24/\pinform$. If $\pinform/(24s_i) > \mu$, which implies $d < \log(\pinform/(24s_i)) \le \log(1/s_i)$, we get an upper bound of
\[
\frac{25 + 24d}{\pinform} + \frac{1}{\mu} \cdot \frac{1}{s_i}
< \frac{25 + 24 \log(\frac{1}{s_i})}{\pinform} + \frac{1}{\mu} \cdot \frac{1}{s_i}.
\]
Otherwise, \eqref{eq:hypercube-time-on-level-i} is minimized for $2^k=\pinform/(24 s_i)$, leading to
\[
\frac{25}{\pinform} + \frac{24 \log(\frac{\pinform}{24s_i})}{\pinform} + \frac{24}{\pinform} \le
\frac{49}{\pinform} + \frac{24 \log(\frac{1}{s_i})}{\pinform}.
\]
The maximum over all these bounds is at most
\[
\frac{49+24\log(\frac{1}{s_i})}{\pinform} + \frac{1}{\mu} \cdot \frac{1}{s_i}.
\]
Summing over all fitness levels yields the claim.
\end{proof}

Results for our example applications are as follows.
\begin{theorem}
The following holds for the parallel \EA with transmission probability~$\pinform$ on a $(\log \mu)$-dimensional hypercube:
\begin{itemize}
\item $\E{\paralleltime} = \mathord{O}\mathord{\left(\frac{n}{\pinform} + \frac{n \log n}{\mu}\right)}$ for \OM{},
\item $\E{\paralleltime} = \mathord{O}\mathord{\left(\frac{d \log n}{\pinform} + \frac{dn}{\mu}\right)}$ for every unimodal function with $d+1$ function values,
\item $\E{\paralleltime} = \mathord{O}\mathord{\left(\frac{n + k \log n}{\pinform} + \frac{n^k}{\mu}\right)}$ for $\Jump_k$ with $k \ge 2$.
\end{itemize}
\end{theorem}
\begin{proof}
For $\OM$ we have
\begin{align*}
\sum_{i=0}^{n-1} \log\left(\frac{1}{s_i}\right) =
\sum_{i=1}^{n} \log\left(\frac{en}{i}\right)
&\;= \log\left(\prod_{i=1}^n \frac{en}{i}\right)\\
&\;= \log\left(\frac{e^n n^n}{n!}\right)
\le \log\left(\frac{e^n n^n}{(n/e)^n}\right)
= \log\left(e^{2n}\right)
= 2n\log\left(e\right).
\end{align*}
Theorem~\ref{the:method-hypercube} gives an upper bound of
\[
\frac{49n+48n \log e}{\pinform}  + \mathord{O}\mathord{\left(\frac{n \log n}{\mu}\right)}
=
\mathord{O}\mathord{\left(\frac{n}{\pinform}  + \frac{n \log n}{\mu}\right)}.
\]
For unimodal functions Theorem~\ref{the:method-hypercube} gives
\[
\frac{49d + 24d \log(en)}{\pinform} + \mathord{O}\mathord{\left(\frac{dn}{\mu}\right)}
= \mathord{O}\mathord{\left(\frac{d \log n}{\pinform} + \frac{dn}{\mu}\right)}.
\]
For $\Jump_k$ we get
\begin{align*}
& \mathord{O}\mathord{\left(\frac{n}{\pinform}  + \frac{n \log n}{\mu}\right)} +
\frac{49+24 k\log(en)}{\pinform} + \mathord{O}\mathord{\left(\frac{n^k}{\mu}\right)}\\
=\;& \mathord{O}\mathord{\left(\frac{n + k \log n}{\pinform} + \frac{n^k}{\mu}\right)}
.
\qedhere
\end{align*}
\end{proof}

If $\pinform = \Omega(1)$, we get linear speedups for \OM if $\mu = O(\log n)$, and linear speedups for unimodal functions where the bound $O(dn)$ for a single island is tight, if $\mu = O(n/\log n)$. For $\Jump_k$, if $k = O(n/\log n)$ we can choose $\mu = O(n^{k-1})$ to get a linear speedup. As can be seen from Table~\ref{tab:running-times-best-mu-p=1} the expected parallel times for \LO and $\Jump_k$ are much better for the hypercube than for rings and torus graphs, if $\pinform$ is large.

\section{Parallel EAs with Complete Topologies}
\label{sec:PerfectParallelization}

Finally, we consider the densest topology, the complete graph $K_{\mu}$, where every island is neighboring to every other island. The complete graph is interesting because it represents an extreme case: the largest possible communication costs, but also the fastest possible spread of information.

For the special case of $\pinform=1$ a parallel \EA is basically equivalent to a (1+$\mu$)~EA, which creates $\mu$ offspring independently and then compares a best offspring against the current search point. The only difference is that the parallel \EA can store different individuals of the same fitness. But this issue is irrelevant when using the fitness-level method. Hence our results for a parallel \EA with a complete topology and $\pinform=1$ also apply for the (1+$\mu$)~EA. For $\pinform < 1$ the two models are generally different.

We start with a simple argument. Clearly, if there is at least one informed island, each other island will become informed with probability at least $\pinform$.
\begin{theorem}
\label{the:method-completegraph}
Consider the setting from Theorem~\ref{the:method-ring}.
The expected parallel running time of the island model on a complete topology is
\[
E(\paralleltime) \;\le\;  m + \frac{2m}{\pinform} +\frac{2}{\mu} \sum_{i=1}^{m-1} \frac{1}{s_i}\;.
\]
The expected communication effort is by a factor of at most $\pcomm \mu^2$ larger than the expected parallel time.
\end{theorem}
\begin{proof}
We estimate the expected time until at least $\mu/2$ islands are informed after an improvement. If more than $\mu/2$ islands are uninformed, the expected number of islands that become informed in one generation is at least $\pinform \mu/2$. By standard drift analysis arguments~\cite{He2004} the desired expectation is bounded by $2/\pinform$.

By Lemma~\ref{lem:fitness-level-k} we then get that the expected time on fitness level~$i$ is at most
\[
1 + \frac{2}{\pinform} + \frac{2}{\mu} \cdot \frac{1}{s_i}.
\]
Adding these times for all fitness levels proves the claim.
\end{proof}

As mentioned, the complete graph leads to a maximal spread of information. In comparison to the previous sections, we obtain the best upper bounds for the considered function classes. However, also the communication effort in one generation is maximal, so the expected total communication costs are also highest (cf.\ Tables~\ref{tab:running-times-best-mu-p=1} and~\ref{tab:running-times-full-throttle}).
\begin{theorem}
\label{the:application-complete}
Let $\mu \in \N$. The following holds for the expected parallel running time of the parallel \EA with topology $K_\mu$. In the case~$\pinform=1$, the same holds for the (1+$\mu$)~EA:
\begin{itemize}
\item $\E{\paralleltime} = \Oh{\frac{n}{\pinform}+\frac{n \log n}{\mu}}$ for \OM{},
\item $\E{\paralleltime} = \Oh{\frac{d}{\pinform}+\frac{dn}{\mu}}$ for every unimodal function with $d+1$ function values, and
\item $\E{\paralleltime} = \Oh{\frac{n}{\pinform}+\frac{n^k}{\mu}}$ for $\Jump_k$ with $k \ge 2$.
\end{itemize}
\end{theorem}
The proof is obvious by now.

The term $2/\pinform$ for the time until at least $\mu/2$ islands are informed is a reasonable estimate if $\pinform$ is large (e.\,g., $\pinform=\Omega(1)$). But for small $\pinform$ this estimation is quite loose as we have completely neglected that \emph{all} informed vertices have a chance to inform other islands.

We therefore also present a more detailed analysis for small~$\pinform$. The motivation for studying complete graphs and small $\pinform$ is that it captures random migration policies. Assume that each island decides randomly with probability $\pcomm$ for each other island whether to migrate individuals to that island.
Then this can be regarded as a complete topology with transmission probability~$\pinform$.

Values around $\pinform = 1/\mu$ seem particularly interesting as then in each generation one migration takes place for each island in expectation. There also a change of regime happens as we get different results for $\pinform > 1/\mu$ and $\pinform < 1/\mu$.
\begin{lemma}
\label{lem:propagation-complete}
Consider propagation with transmission probability~$\pinform$ on the complete topology with $\mu$ vertices.
Let $\Tinform{k}$ be as in Lemma~\ref{lem:fitness-level-k}, then
\[
\Tinform{\mu} \le \frac{8\log(\mu)}{\min(\pinform \mu, 1)}.
\]
\end{lemma}
\begin{proof}
Let $X_t$ denote the random number of informed vertices after $t$ iterations. We first estimate the expected time until at least $\mu/2$ vertices become informed, and then estimate how long it takes to get from $\mu/2$ informed vertices to $\mu$ informed ones.

If $X_t = i$ each presently uninformed vertex is being informed in one iteration with probability (using Lemma~\ref{lemma:parallel-bounds})
\[
1-(1-p)^{i}
\ge 1 - \frac{1}{1+ip} = \frac{ip}{1+ip} =: iq.
\]
This holds independently from other presently uninformed vertices.
In fact, the number of newly informed vertices follows a binomial distribution with parameters $\mu-i$ and~$iq$.
The median of this binomial distribution is $i(\mu-i)q$ (assuming that this is an integer), hence with probability at least $1/2$ we have at least $i(\mu-i)q$ newly informed vertices in one iteration. Hence, it takes an expected number of at most 2 iterations to increase the number of informed vertices by $i(\mu-i) \cdot \frac{p}{1+ip}$, which for $i \le \mu/2$ is at least $i \cdot \frac{\pinform \mu}{2+\pinform \mu}$.

For every $0 \le j \le \log(\mu)-2$ the following holds. If $i \ge 2^j$ then in an expected number of 2 generations at least $2^j \cdot \frac{\pinform \mu}{2+\pinform \mu}$ new vertices are informed. The expected number of iterations for informing a total of $2^j$ new vertices is therefore at most $2 \cdot \frac{2+\pinform \mu}{\pinform \mu}$. Then we have gone from at least $2^j$ informed vertices to at least $2^{j+1}$ informed vertices. Summing up all times across all~$j$, the expected time until at least $2^{\log(\mu)-1} = \mu/2$ vertices are informed is at most
\[
2(\log(\mu)-1) \cdot \frac{2+\pinform \mu}{\pinform \mu}.
\]
For $\pinform \mu \le 1$ we have $\frac{2+\pinform \mu}{\pinform \mu} \le 3/(\pinform \mu)$, yielding an upper time bound of~$6 (\log(\mu)-1)/(\pinform \mu)$. Otherwise, we use $\frac{2+\pinform \mu}{\pinform \mu} \le 3$ to get a bound of~$6 (\log(\mu)-1)$.

For the time to get from $\mu/2$ to $\mu$ informed vertices, observe that the expected number of newly informed vertices is $\frac{i(\mu-i)p}{1+ip}$, if currently $i$ vertices are informed. This function is monotone decreasing if $i \ge \mu/2$. Applying Johannsen's drift theorem, Theorem~\ref{drift:johannsen}, for the number of uninformed nodes, using the above as drift, gives an upper bound of
\begin{align*}
& \frac{1+(\mu-1)p}{(\mu-1)p} + \int_1^{\mu/2} \frac{1+ip}{i(\mu-i)p} \;\mathrm{d}i\\
\le\;& \frac{1+(\mu-1)p}{(\mu-1)p} + \frac{\ln(\mu-1)(1+\pinform \mu)}{\pinform \mu}\\
\le\;& \frac{1+\pinform \mu}{\pinform \mu} + \frac{1}{\mu(\mu-1)p} + \frac{\ln(\mu-1)(1+\pinform \mu)}{\pinform \mu}\\
\le\;& \frac{(\ln(\mu)+1)(1+\pinform \mu) + \frac{1}{\mu-1}}{\pinform \mu}.
\end{align*}
For $\pinform \mu \le 1$ this is at most
\[
\frac{2\ln(\mu) + 2 + \frac{1}{\mu-1}}{\pinform \mu} \le \frac{2\ln(\mu) + 5/2}{\pinform \mu}.
\]
Otherwise, this is at most
\[
\frac{(\ln(\mu)+1) \cdot 2\pinform \mu + \frac{\pinform \mu}{\mu-1}}{\pinform \mu}
\le 2\ln(\mu) + 5/2.
\]
Together, along with $\ln(\mu) \le \log(\mu)$ this proves the claim.
\end{proof}

Combining Lemma~\ref{lem:propagation-complete} with Lemma~\ref{lem:fitness-level-k} gives the following.
Apart from an additive term~$m$, the case of $\pinform \le 1/\mu$ yields a bound where the first term is smaller by a factor of order $\log(\mu)/\mu$. For fairly large transmission probabilities, $\pinform \ge 1/\mu$, in the first term we have replaced the factor $1/\pinform$ by $\log(\mu)$. These improvements reflect that the complete graph can spread information much more quickly than previously estimated in the proof of Theorem~\ref{the:method-completegraph}.
\begin{theorem}
\label{the:method-completegraph-refined}
Consider the setting from Theorem~\ref{the:method-ring}.
The expected parallel running time of the island model on a complete topology is bounded as follows.
If $\pinform \ge 1/\mu$ we have
\[
E(\paralleltime) \;\le\;  m + 8m\log \mu + \frac{1}{\mu} \sum_{i=1}^{m-1} \frac{1}{s_i}\;
\]
and if $\pinform \le 1/\mu$ we have
\[
E(\paralleltime) \;\le\;  m + \frac{8m\log \mu}{\pinform \mu} + \frac{1}{\mu} \sum_{i=1}^{m-1} \frac{1}{s_i}\;.
\]
\end{theorem}

For our example applications, the refinements in Theorem~\ref{the:method-completegraph-refined} result in the following refined bounds. As we only get improvements for $\pinform = O(1/\log(\mu))$, we do not mention the special case of the (1+$\mu$)~EA with $\pinform=1$.
\begin{theorem}
\label{the:application-complete}
Let $\mu \in \N$. The following holds for the expected parallel running time of the parallel \EA with topology $K_\mu$:
\begin{itemize}
\item $\E{\paralleltime} = \Oh{n \log(\mu) +\frac{n \log n}{\mu}}$ for \OM{} if $\pinform \ge 1/\mu$ and \newline
    $\E{\paralleltime} = \Oh{n + \frac{n \log \mu}{\pinform \mu} +\frac{n \log n}{\mu}}$ otherwise,
\item $\E{\paralleltime} = \Oh{d \log(\mu) +\frac{dn}{\mu}}$ for unimodal functions with $d+1$ values, if $\pinform \ge 1/\mu$, and \newline
     $\E{\paralleltime} = \Oh{d + \frac{d \log(\mu)}{\pinform \mu} +\frac{dn}{\mu}}$ otherwise, and
\item $\E{\paralleltime} = \Oh{n \log(\mu) +\frac{n^k}{\mu}}$ for $\Jump_k$ with $k \ge 2$, if $\pinform \ge 1/\mu$ and \newline
    $\E{\paralleltime} = \Oh{n + \frac{n \log(\mu)}{\pinform \mu} +\frac{n^k}{\mu}}$ otherwise.
\end{itemize}
\end{theorem}

\section{Experiments}
\label{sec:Experiments}

In order to complement the analytical results above, we also give experimental results on the behavior of island models for different topologies. As a detailed experimental evaluation is beyond the scope of this paper, we only present illustrative results for the two functions \OM and \LO.

First we investigate the parallel running time $\paralleltime$ for different transmission probabilities. The experiments were repeated 100 times per data point for the parallel \EA with $\mu=64$ islands and an instance size of $n=256$ for all example functions, varying the transmission probability $\pinform$ in steps of 0.01. Figure~\ref{fig:transmissionProbabilities} shows the behavior for the topologies $K_{64}$, a bidirectional ring graph, an $8 \times 8$ torus graph, and a $6$-dimensional hypercube.

{
\tikzset{/pgfplots/no markers}
\begin{figure}[hbt]
\centering
\tikzsetnextfilename{OM-different-topologies-differentPplus}
\subfigure[\OM]{\includegraphics{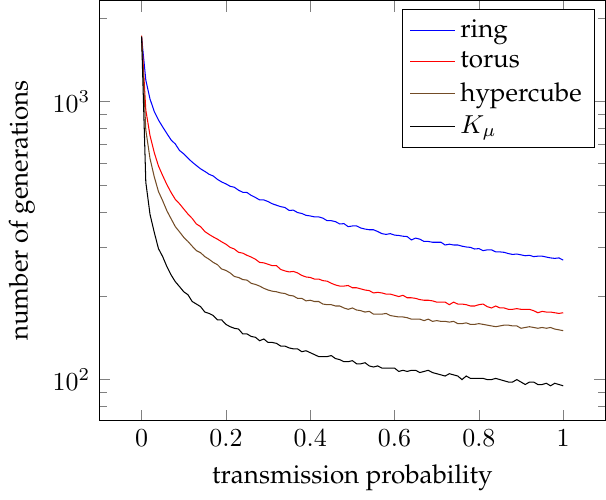}}
\tikzsetnextfilename{LO-different-topologies-differentPplus}
\subfigure[\LO]{\includegraphics{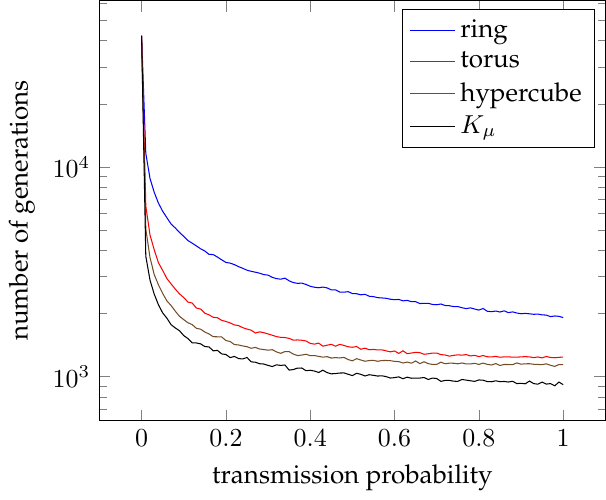}}
\caption{Average parallel running time for the parallel \EA using $\mu=64$ islands and different transmission probabilities, both for \OM and \LO on $n=256$ bits.}
\label{fig:transmissionProbabilities}
\end{figure}
}

Looking at the influence of the transmission probability on the running time, a higher transmission probability improves the running time behavior of the algorithm, also according to the expectations from our theoretical analysis. In particular, all not too small $\pinform$ lead to much smaller running times compared to the setting $\pinform=0$, i.\,e., $\mu$ independent runs of the \EA. This demonstrates for our functions that parallelization and migration can lead to drastic speedups. For larger or intermediate values for $\pinform$ the parallel running time does not vary much, as then for all topologies the running time is dominated by the second terms from our bounds: $1/\mu \cdot O(n \log n)$ and $1/\mu \cdot O(n^2)$ for \OM and \LO, respectively.

Comparing the behavior of those topologies, we see that the parallel running time indeed depends on the density of the topology, i.\,e., more dense topologies spread information more efficiently, which results in a faster convergence. As expected, the topology $K_\mu$ performs best, the ring graph performs worst.

\pgfplotsset{every linear axis/.append style={xmax=64}}
{
\begin{figure}[tbh]
\centering
\tikzsetnextfilename{OM-different-topologies-p=1.0-efficiency}
\subfigure[Efficiency for \OM with $\pinform=1.0$\label{fig:transmissionProb10a}]{\includegraphics{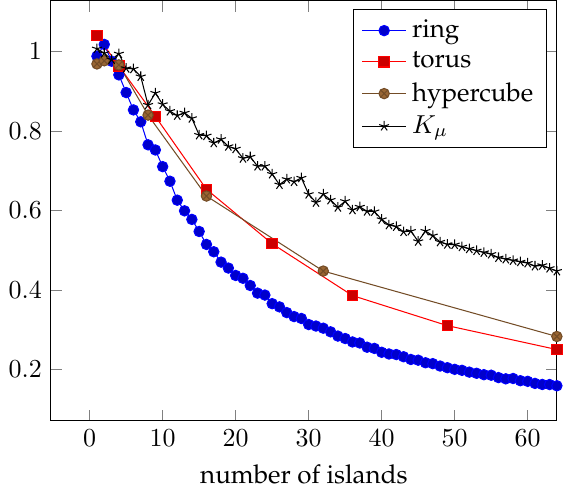}}
\tikzsetnextfilename{LO-different-topologies-p=1.0-efficiency5}
\subfigure[Efficiency for \LO with $\pinform=1.0$\label{fig:transmissionProb10b}]{\includegraphics{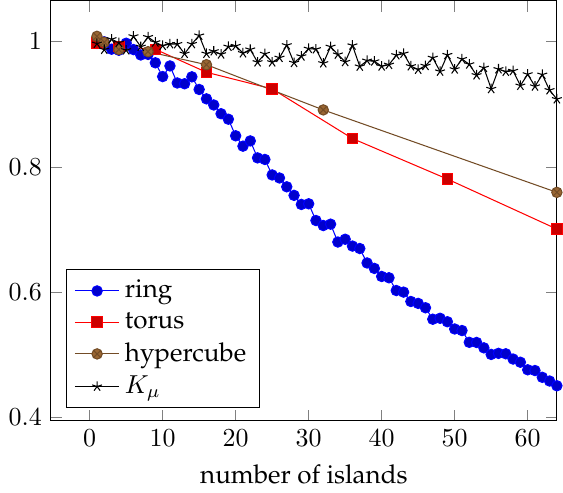}}
\tikzsetnextfilename{OM-different-topologies-p=0.1-efficiency}
\subfigure[Efficiency for \OM with $\pinform=0.1$\label{fig:transmissionProb01a}]{\includegraphics{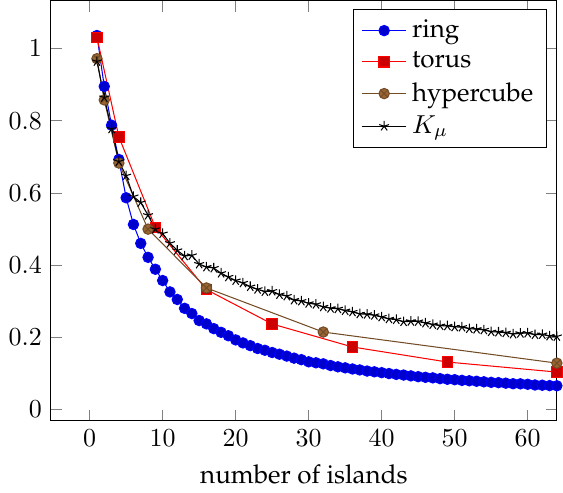}}
\tikzsetnextfilename{LO-different-topologies-p=0.1-efficiency5}
\subfigure[Efficiency for \LO with $\pinform=0.1$\label{fig:transmissionProb01b}]{\includegraphics{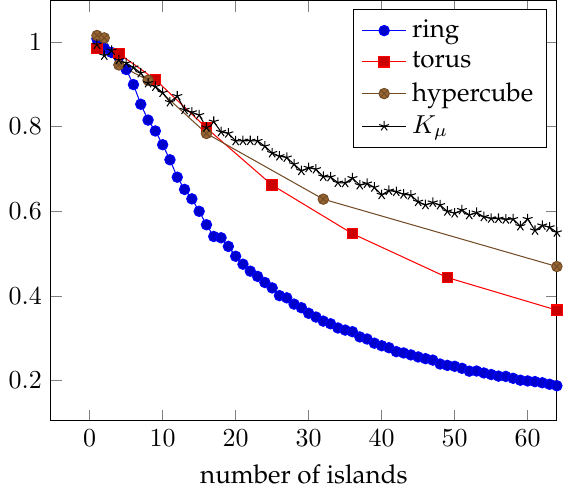}}
\caption{Efficiency for the parallel \EA with transmission probabilities $\pinform \in \{0.1, 1\}$ for $\mu \in \{1, \dots, 64\}$ numbers of islands.}
\label{fig:efficiency01}
\end{figure}
}

Next we investigate the impact of the number of islands on performance, with regard to different topologies and transmission probabilities, see Figures~\ref{fig:transmissionProb10a} and~\ref{fig:transmissionProb10b} for a transmission probability $\pinform=1.0$ and Figures~\ref{fig:transmissionProb01a} and~\ref{fig:transmissionProb01b} for a transmission probability $\pinform=0.1$. As the parallel running time shows a steep decrease, we plot the \emph{efficiency} instead, defined as
\[
\frac{\sequentialtime}{\paralleltime\cdot \mu}\;.
\]
It can be regarded as a normalized version of speedup, normalized by the number of islands. Small efficiencies indicate small speedups, large efficiencies indicate good speedups. An efficiency of~1 corresponds to a perfect linear speedup.

Again, the instance size of the benchmark functions was set to $n=256$ and the number of islands $\mu$ was chosen from 1 to 64. Only square torus graphs were used. So our torus graphs and hypercubes are only defined for square numbers and powers of~2, respectively, leading to fewer data points. For lower numbers of islands the efficiency of the algorithm is better than for larger numbers of islands. This is somewhat expected as a single \EA, i.\,e., our setting with $\mu=1$, minimizes the number of function evaluations for both \OM and \LO~\cite{Sudholt2012c}, among all EAs that only use standard bit mutation. This excludes superlinear speedups on \OM and \LO, for such EAs.

It can be seen that more dense topologies are more efficient than sparse topologies. Also, the efficiency is decreasing with a higher number of islands. In accordance with our theoretical analyses, the efficiency decreases more rapidly for \OM. For \OM, and $\pinform = \Omega(1)$, only values $\mu=O(\log n)$ were guaranteed to give a linear speedup. And indeed the efficiency in Figure~\ref{fig:transmissionProb01a} degrades quite quickly for \OM and $\pinform=1$.

Higher numbers of islands are still efficient for \LO. For the ring, the range of good $\mu$-values is up to $\mu=O(\sqrt{n})$. This is reflected in Figure~\ref{fig:transmissionProb01b} as the efficiency degrades as $\mu$ increases beyond $\sqrt{n} = 16$. For denser topologies the efficiency only degrades for large~$\mu$. The complete graph remains effective throughout the whole scale---even stronger, for values up to~$\mu=256$ (not shown in Figure~\ref{fig:efficiency01}) the efficiency was always above~$0.75$. This was also expected as $\mu = \Theta(n)$ still guarantees a linear speedup for \LO.

Comparing the running time behavior for different transmission probabilities, the plots confirm again that in our examples a higher transmission probability for individuals allows for a better overall performance.

\section{Conclusions}
\label{sec:Conclusion}
We have provided a new method for the running time analysis of parallel evolutionary algorithms, including applications to a set of well-known and illustrative example functions.
Our method provides a way of automatically transforming running time bounds obtained for panmictic EAs to parallel EAs with spatial structures.
In addition to a general result, we have provided methods tailored towards specific topologies: ring graphs, torus graphs, hypercubes and complete graphs. The latter also covers offspring populations and random migration topologies as special cases. Our results can estimate the expected parallel running time, and hence the speedup obtained through parallelization. They also bound the expected total communication effort in terms of the total number of individuals migrated.

Our example applications revealed insights which are remarkable in their own right, see Table~\ref{tab:running-times-best-mu-p=1} and a more general version in Table~\ref{tab:running-times-full-throttle}. Compared to upper bounds obtained for a single panmictic island by the fitness-level method, for ring graphs the expected waiting time for an improvement can be replaced by its square root in the parallel running time, provided the number of islands is large enough and improvements are transmitted efficiently, i.\,e., $\pinform = \Omega(1)$. This leads to a speedup of order $\log n$ for \OM{} and of order $\sqrt{n}$ for unimodal functions like \LO{}. On $\Jump_k$ the speedup is even of order $n^{k/2}$. A similar effect is observed for torus graphs where the expected waiting time can be replaced by its third root. The hypercube reduces the expected waiting time on each level to its logarithm, and on the complete graph it is reduced to a constant, again provided there are sufficiently many islands. This way, even on functions like \LO{} and $\Jump_k$ ($3 \le k = O(n/\log n)$) the expected parallel time can be reduced to $O(n)$. In all these results the population size can be chosen in such a way that the total number of function evaluations does not increase, in an asymptotic sense. The ``optimal'' population sizes have been stated explicitly (cf.\ Tables~\ref{tab:running-times-best-mu-p=1} and~\ref{tab:running-times-full-throttle}), therefore giving hints on how to parametrize parallel EAs.

The tables also reveal that in certain situations there is a tradeoff between the expected parallel time and the communication effort. For instance, on \LO the torus graph has the smallest communication effort of $O(n^2)$ at the expense of a higher parallel time bound of $O(n^{4/3})$. The complete graph has the smallest bound for the parallel time, $O(n)$, but the largest communication costs: $O(n^3)$. The hypercube provides a good compromise, combining the smallest bounds up to polylogarithmic factors. A similar observation can be made for $\Jump_k$, but there the hypercube is the better choice than the complete graph (strictly better in terms of communication costs and equally good in the parallel time bound). In all our examples the ring never performed better than torus graphs in both objectives.

Future work should deal with lower bounds on the running time of parallel evolutionary algorithms. Also in our example functions no diversity was needed. Further studies are needed in order to better understand how the topology and the parameters of migration affect diversity, and how diversity helps for optimizing more difficult, multimodal problems.

\subsection*{Acknowledgments}
The authors were supported by postdoctoral fellowships from the German Academic Exchange Service while visiting the International Computer Science Institute, Berkeley, CA, USA. Dirk Sudholt was also supported by EPSRC
grant EP/D052785/1.

\bibliographystyle{abbrv}
\newcommand{\noopsort}[1]{} \newcommand{\printfirst}[2]{#1}
  \newcommand{\singleletter}[1]{#1} \newcommand{\switchargs}[2]{#2#1}

\appendix

\section{Appendix}


The following inequality was brought to our attention by Jon Rowe. A proof is found in~\cite[Lemma~3]{Rowe2012}.
\begin{lemma}\label{lemma:parallel-bounds}
For any $0 \leq x \leq 1$, and any $n > 0$
\[
	(1 - x)^n \leq \frac{1}{1 + nx}.
\]
\end{lemma}

We also state Johannsen's variable drift theorem~\cite{Johannsen2010}, in a version with slightly improved conditions~\cite{Rowe2012}.
\begin{theorem}[Johannsen's Variable Drift Theorem~\cite{Johannsen2010,Rowe2012}]
\label{drift:johannsen}
Consider a stochastic process $\{X\}_{t \ge 0}$ on $\{0, 1, \dots, m\}$, with $m \in \N$.
Suppose there is a monotonic increasing function $h: \mathbb{R}^+ \rightarrow \mathbb{R}^+$ such that the function $1/h(x)$ is integrable on $\{1, \dots,  m\}$, and with
\[
	\E{X_t - X_{t+1} \mid X_t = k} \geq h(k)
\]
for all $k \in \{1, \ldots, m\}$. Then the expected first hitting time of state~$0$ is at most
\[
	\frac{1}{h(1)} + \int_1^m \frac{1}{h(x)} \;\mathrm{d}x.
\]
\end{theorem}

\begin{landscape}
\begin{table}[h!]
\centering\small
\begin{tabular}{|l@{\;}|@{\;}l@{\;}|@{\;}l@{\;}|@{\;}l@{\;}|@{\;}l@{\;}|@{\;}l@{\;}|@{\;}l@{\;}|}\hline
                     & \EA               & Ring            & Grid/Torus            & Hypercube   & Complete/$K_{\mu}$ & $K_{\mu}$ with $\pinform = O(1/\mu)$\\\hline
$\OM$     &&&&&&\\
$\E{\paralleltime}$ & $\Theta(n\log n)$ & $\Oh{\frac{n}{\pinform^{1/2}} + \frac{n \log n}{\mu}}$                   & $\Oh{\frac{n}{\pinform^{2/3}} + \frac{n \log n}{\mu}}$       & $\Oh{\frac{n}{\pinform} + \frac{n \log n}{\mu}}$            & $\Oh{\frac{n}{\pinform}+\frac{n \log n}{\mu}}$ & $\Oh{\frac{n \log \mu}{\pinform \mu}+\frac{n \log n}{\mu}}$\\
$\E{\sequentialtime}$ & $\Theta(n\log n)$ & $\Oh{\frac{\mu n}{\pinform^{1/2}} + n \log n}$                   & $\Oh{\frac{\mu n}{\pinform^{2/3}} + n \log n}$       & $\Oh{\frac{\mu n}{\pinform} + n \log n}$            & $\Oh{\frac{\mu n}{\pinform}+ n \log n}$ & $\Oh{\frac{\mu (\log \mu) n}{\pinform \mu}+ n \log n}$\\
$\E{\communication}$ & $0$ & $\Oh{\pinform^{1/2} \mu n + \pinform n \log n}$                   & $\Oh{\pinform^{1/3} \mu n + \pinform n \log n}$       & {\tiny{}$\Oh{\mu n (\log \mu) + \pinform (\log \mu) n \log n}$}            & $\Oh{\mu^2 n + \pinform \mu n \log n}$ & $\Oh{\mu (\log \mu) n + \pinform \mu n \log n}$\\\hline
$\LO$     &&&&&&\\
$\E{\paralleltime}$ & $\Theta(n^2)$     &  $\Oh{\frac{n^{3/2}}{\pinform^{1/2}} + \frac{n^2}{\mu}}$             & $\Oh{\frac{n^{4/3}}{\pinform^{2/3}} + \frac{n^2}{\mu}}$      & $\Oh{\frac{n \log n}{\pinform} + \frac{n^2}{\mu}}$       & $\Oh{\frac{n}{\pinform}+\frac{n^2}{\mu}}$ & $\Oh{\frac{n \log \mu}{\pinform \mu}+\frac{n^2}{\mu}}$\\
$\E{\sequentialtime}$ & $\Theta(n^2)$     &  $\Oh{\frac{\mu n^{3/2}}{\pinform^{1/2}} + n^2}$             & $\Oh{\frac{\mu n^{4/3}}{\pinform^{2/3}} + n^2}$      & $\Oh{\frac{\mu n \log n}{\pinform} + n^2}$       & $\Oh{\frac{\mu n}{\pinform}+n^2}$ & $\Oh{\frac{\mu (\log \mu) n}{\pinform \mu}+ n^2}$\\
$\E{\communication}$ & $0$     &  $\Oh{\pinform^{1/2} \mu n^{3/2} + \pinform n^2}$             & $\Oh{\pinform^{1/3} \mu n^{4/3} + \pinform n^2}$      & {\tiny{}$\Oh{\mu n (\log n)(\log \mu) + \pinform (\log \mu) n^2}$}       & $\Oh{\mu^2 n + \pinform \mu n^2}$ & $\Oh{\mu (\log \mu) n + \pinform \mu n^2}$\\\hline
unimodal &&&&&&\\
$\E{\paralleltime}$ & $O(dn)$     &  $\Oh{\frac{d n^{1/2}}{\pinform^{1/2}} + \frac{dn}{\mu}}$             & $\Oh{\frac{dn^{1/3}}{\pinform^{2/3}} + \frac{dn}{\mu}}$     & $\Oh{\frac{d \log n}{\pinform} + \frac{dn}{\mu}}$        & $\Oh{\frac{d}{\pinform}+\frac{dn}{\mu}}$ & $\Oh{\frac{d \log \mu}{\pinform \mu}+\frac{dn}{\mu}}$\\
$\E{\sequentialtime}$ & $O(dn)$     &  $\Oh{\frac{d \mu n^{1/2}}{\pinform^{1/2}} + dn}$             & $\Oh{\frac{d \mu n^{1/3}}{\pinform^{2/3}} + dn}$     & $\Oh{\frac{d \mu \log n}{\pinform} + dn}$        & $\Oh{\frac{d \mu}{\pinform}+dn}$ & $\Oh{\frac{d \log \mu}{\pinform \mu}+dn}$\\
$\E{\communication}$ & $0$     &  $\Oh{\pinform^{1/2} d \mu n^{1/2} + \pinform dn}$             & $\Oh{\pinform^{1/3} d \mu n^{1/3} + \pinform dn}$     & {\tiny{}$\Oh{d \mu (\log n)(\log \mu) + \pinform (\log \mu) dn}$}        & $\Oh{d \mu^2 + \pinform d \mu n}$ & $\Oh{d \mu \log \mu + \pinform \mu d n}$\\\hline
$\Jump_k$ &&&&&&\\
$\E{\paralleltime}$ & $\Theta(n^k)$     & $\Oh{\frac{n^{k/2}}{\pinform^{1/2}} + \frac{n^k}{\mu}}$             & $\Oh{\frac{n + n^{k/3}}{\pinform^{2/3}} + \frac{n^k}{\mu}}$    & $\Oh{\frac{n}{\pinform} + \frac{n^k}{\mu}}$       & $\Oh{\frac{n}{\pinform}+\frac{n^k}{\mu}}$ & $\Oh{\frac{n \log \mu}{\pinform \mu}+\frac{n^k}{\mu}}$\\
$\E{\sequentialtime}$ & $\Theta(n^k)$     & $\Oh{\frac{\mu n^{k/2}}{\pinform^{1/2}} + n^k}$             & $\Oh{\frac{\mu (n + n^{k/3})}{\pinform^{2/3}} + n^k}$    & $\Oh{\frac{\mu n}{\pinform} + n^k}$       & $\Oh{\frac{\mu n}{\pinform}+n^k}$ & $\Oh{\frac{n \log \mu}{\pinform}+n^k}$\\
$\E{\communication}$ & $0$     & $\Oh{\pinform^{1/2} \mu n^{k/2} + \pinform n^k}$             & $\Oh{\pinform^{1/3} \mu n^{k/3} + \pinform n^k}$    & {\tiny{}$\Oh{\mu n (\log \mu) + \pinform (\log \mu) n^k}$}       & $\Oh{\mu^2 n + \pinform \mu n^k}$ & $\Oh{\mu (\log \mu) n + \pinform \mu n^k}$\\
\hline
\end{tabular}\smallskip
\caption{Asymptotic bounds on expected parallel ($\paralleltime$, number of generations) and sequential ($\sequentialtime$, number of function evaluations) running times and expected communication efforts ($\communication$, total number of migrated individuals) for various $n$-bit functions and island models with $\mu$ islands running the \EA and using migration probability~$\pinform$. The number of islands $\mu$ was always chosen to give the best possible upper bound on the parallel running time, while not increasing the upper bound on the sequential running time by more than a constant factor. For unimodal functions $d+1$ denotes the number of function values. See~\cite{Droste2002} for bounds for the \EA. Results for $\Jump_k$ were restricted to $3 \le k = O(n/\log n)$ for simplicity.}
\label{tab:running-times-full-throttle}
\end{table}
\end{landscape}

\end{document}